\documentclass[letterpaper, 10 pt, conference]{ieeeconf}
\usepackage{graphicx} 
\usepackage{subcaption} 
\usepackage{xcolor}
\usepackage{soul}
\usepackage{makecell}
\usepackage{algpseudocode}
\usepackage{algorithm}
\usepackage{cite}
\usepackage{amsmath}
\usepackage{siunitx}
\usepackage{amssymb}
\usepackage{pgf}
\usepackage{lmodern}
\usepackage{import}
\usepackage{url}

\usepackage{amsthm}
\renewenvironment{proof}[1][Sketch]{%
  \par\noindent\textbf{Proof}\ \textit{(#1): }%
}{\hfill$\square$\par}
\newtheorem{definition}{Definition}
\newtheorem{lemma}{Lemma}

\newtheorem{proposition}{Proposition}

\usepackage{colortbl}
\definecolor{lightgray}{gray}{0.92}

\title{Probabilistic Collision Risk Estimation through Gauss-Legendre Cubature and Non-Homogeneous Poisson Processes}
\author{Trent Weiss and Madhur Behl}

\date{September 2024}

\begin{document}
\setlength{\textfloatsep}{0.3cm}
\setlength{\floatsep}{0.3cm}
\maketitle
\begin{abstract}
    Overtaking in high-speed autonomous racing demands precise, real-time estimation of collision risk; particularly in wheel-to-wheel scenarios where safety margins are minimal. Existing methods for collision risk estimation either rely on simplified geometric approximations, like bounding circles, or perform Monte Carlo sampling which leads to overly conservative motion planning behavior at racing speeds.
    We introduce the Gauss–Legendre Rectangle (GLR) algorithm, a principled two-stage integration method that estimates collision risk by combining Gauss–Legendre with a non-homogeneous Poisson process over time. 
    GLR produces accurate risk estimates that account for vehicle geometry and trajectory uncertainty.
    In experiments across 446 overtaking scenarios in a high-fidelity Formula One racing simulation, GLR outperforms five state-of-the-art baselines achieving an average error reduction of $77\%$ and surpassing the next-best method by $52\%$, all while running at 1000 Hz. 
    The framework is general and applicable to broader motion planning contexts beyond autonomous racing.
\end{abstract}


\section{Introduction}
\label{sec:intro}
Autonomous racing competitions such as the Indy Autonomous Challenge~\cite{indy} and A2RL have established high-speed autonomous driving as an emerging research frontier \cite{betz2022autonomous}, where vehicles operate close to their dynamic limits.
Among the most challenging maneuvers in multi-agent autonomous racing is autonomously overtaking an opponent, which requires balancing tight track-bounds constraints, dynamic feasibility of the maneuver, and collision risk with other vehicles. 
Because a collision at racing speeds can be catastrophic for the ego and opponent vehicles, precisely estimating the probability of collision over the full planned maneuver is critical for any motion planner aiming to ensure safety.

This brings us to the core probabilistic collision risk estimation problem: given a planned trajectory for the ego vehicle and a predicted trajectory for the opponent, how do we determine the total probability of collision throughout the maneuver? 
In autonomous racing, this is challenging because the opponent’s future position is uncertain and no closed-form solution exists for collision risk across a high-speed trajectory. 
Most existing sampling-based collision risk estimation approaches focus on lower-speed or urban driving, where larger safety margins and more gradual interactions are typical~\cite{lefevre_survey,sciencedirect_survey}.
Vehicles are often approximated with minimal bounding circles - an assumption that tends to overestimate risk in autonomous racing, where margins are both small and significant. 
This can lead to overly cautious or aborted maneuvers, sacrificing overtaking opportunities and increasing lap times, or leading to constant re-planning.
Other approaches~\cite{stochastic_reachability_spaces,mutual_awareness} evaluate only instantaneous collision likelihood rather than the cumulative risk over time, limiting their applicability in high-speed settings.
We propose the Gauss–Legendre Rectangle (GLR) algorithm for probabilistic collision risk estimation.
Rather than simplifying vehicle geometry to a bounding circle, GLR retains the full rectangular shape and evaluates collision risk using a two-stage process:
it first computes the instantaneous collision probability at a fixed point in time using Gauss–Legendre cubature, and then integrates this risk over time using a non-homogeneous Poisson process to yield the total probability of collision across the maneuver.

Our key contributions are:
\begin{enumerate}
    \item A novel probabilistic collision-risk estimator (GLR) that uses spatial Gauss–Legendre cubature and temporal integration of a non-homogeneous Poisson process.
    \item A theoretical formulation that avoids conservative Boole-inequality upper bounds and does not rely on Monte Carlo sampling.
    \item A new dataset of 446 high-speed overtaking scenarios from a high-fidelity Formula One Deepracing simulator, along with a reproducible evaluation pipeline showing that GLR achieves a $77\%$ average error reduction over five five state-of-the-art methods while maintaining a 1000 Hz runtime.
\end{enumerate}

Although motivated by autonomous racing, the GLR framework is applicable to any motion planning context requiring precise risk estimation under uncertainty.


\section{Related Work}
\label{sec:related_work}
Collision risk estimation has been extensively studied, with comprehensive surveys provided in~\cite{lefevre_survey,sciencedirect_survey}.

\noindent \textbf{Risk-aware Motion Planning:} 
Many planners treat collision probability as a constraint in an optimization loop. 
Chance-constrained MPC~\cite{chance_constrained_blackmore} and its variants~\cite{cvar} bound the maximum pointwise collision probability along a trajectory. 
Sampling-based relaxations~\cite{janson2018monte} and reachable set approximations~\cite{dawson2023chance} have also been proposed.
Frey et al.~\cite{Frey-RSS-20} avoid sampling by analytically upper bounding the total risk via Boole’s inequality. 
While these methods work well in low-speed urban driving, or pedestrian avoidance, they tend to be overly conservative for high-speed racing where aggressive maneuvers and precise risk estimation are essential.

\noindent \textbf{Direct Collision Probability Estimation:} 
Other works focus directly on estimating the total collision probability over a trajectory. 
Otte et al.~\cite{ICRA2012_positiononly} and spacecraft conjunction analysis methods~\cite{satellite_spheres} define collision via overlapping bounding spheres.
This overestimates risk in wheel-to-wheel racing where rectangular vehicle footprints often do not collide even when their enclosing circles do.
Risk density methods~\cite{paiola2024riskdensity} and mutual awareness models~\cite{mutual_awareness} measure risk over short intervals or discrete instants but do not integrate over the full maneuver horizon.
Some approaches assume vehicles follow lane constraints~\cite{lanebased} or treat obstacles as static~\cite{swept_path}, assumptions that break down in dynamic overtaking.

\noindent \textbf{Bounding Approximations and Boole’s Inequality:} Finally, many methods (~\cite{discounted_blub,chance_constrained_blackmore,chanceconstrainted_arxiv,Frey-RSS-20,velocity_scaled_pf}) in collision risk estimation utilize Boole's Inequality to place an upper bound on collision probability. I.e. these techniques utilize the fact that the sum of individual collision probabilities is always greater than their mutual product resulting in an over-approximation of the collision risk. 
While appropriate for pedestrian scenarios, such conservative bounds impede overtaking success in racing (see Table~\ref{table:mainresults}).
We group methods using these simplifications under the term \textit{BIUB} (Boole's Inequality Upper Bound).

\noindent \textbf{Summary}
In contrast to prior approaches that rely on bounding-circle approximations or conservative analytical bounds (e.g., Boole’s inequality), our proposed method retains the full rectangular vehicle geometry and models collision risk over a continuous time horizon. 
By combining Gauss–Legendre cubature with a non-homogeneous Poisson process, GLR achieves accurate, parameter-free risk estimation with lower conservatism than sphere-based or BIUB methods.
(Sec.~\ref{sec:results}).
The next section formalizes the problem setting.


\section{Problem Formulation}
\label{sec:problem_formulation}
We consider the future motion of two vehicles:
\begin{enumerate}
    \item \textbf{Ego vehicle}: controlled by an autonomous agent
    \item \textbf{Target vehicle}: an opponent that the ego must overtake
\end{enumerate}

We denote the ego vehicle’s planned trajectory by
\begin{equation}
        \mathcal{T}_{ego}: [0,\, T_F] \,\to\, \mathbb{R}^2,
        \label{eqn:T_ego}
\end{equation}
where \(t = 0\) is the current time, \(t = T_F\) is a future prediction horizon, and \(\mathcal{T}_{ego}(t) = \mathbf{x}_{ego} \in \mathbb{R}^2.\), is the ego vehicle's cartesian position at time $t$.
This \(\mathcal{T}_{ego}\) is the single candidate trajectory produced by the ego’s motion planner; it is \emph{not} stochastic.


By contrast, the target vehicle’s future motion is uncertain, so we model it as a \emph{distribution of trajectories}, $p(\mathcal{T})$,
where each element of \(p(\mathcal{T})\) is a function \([0,\, T_F] \to \mathbb{R}^2.\) 
Figure~\ref{fig:problem_statement} illustrates this setup. The blue-shaded region indicates all possible target-vehicle positions under the distribution \(p(\mathcal{T}).\)  
At each fixed time \(t\), we assume the target vehicle's centroid distribution
\begin{equation}
    p_t(\mathbf{x}) \;=\; p\bigl(\mathbf{x} \,\big\vert\, p(\mathcal{T}),\, t\bigr)
\end{equation}
is known and has a closed-form PDF with mean \(\boldsymbol{\mu}_t\) and covariance \(\boldsymbol{\Sigma}_t\).
Many prediction models (e.g.,~\cite{mtr++,BARTe,densetnt}) can produce such time-indexed probability distributions.




\begin{figure}
   \centering
   \includegraphics[width=0.825\columnwidth]{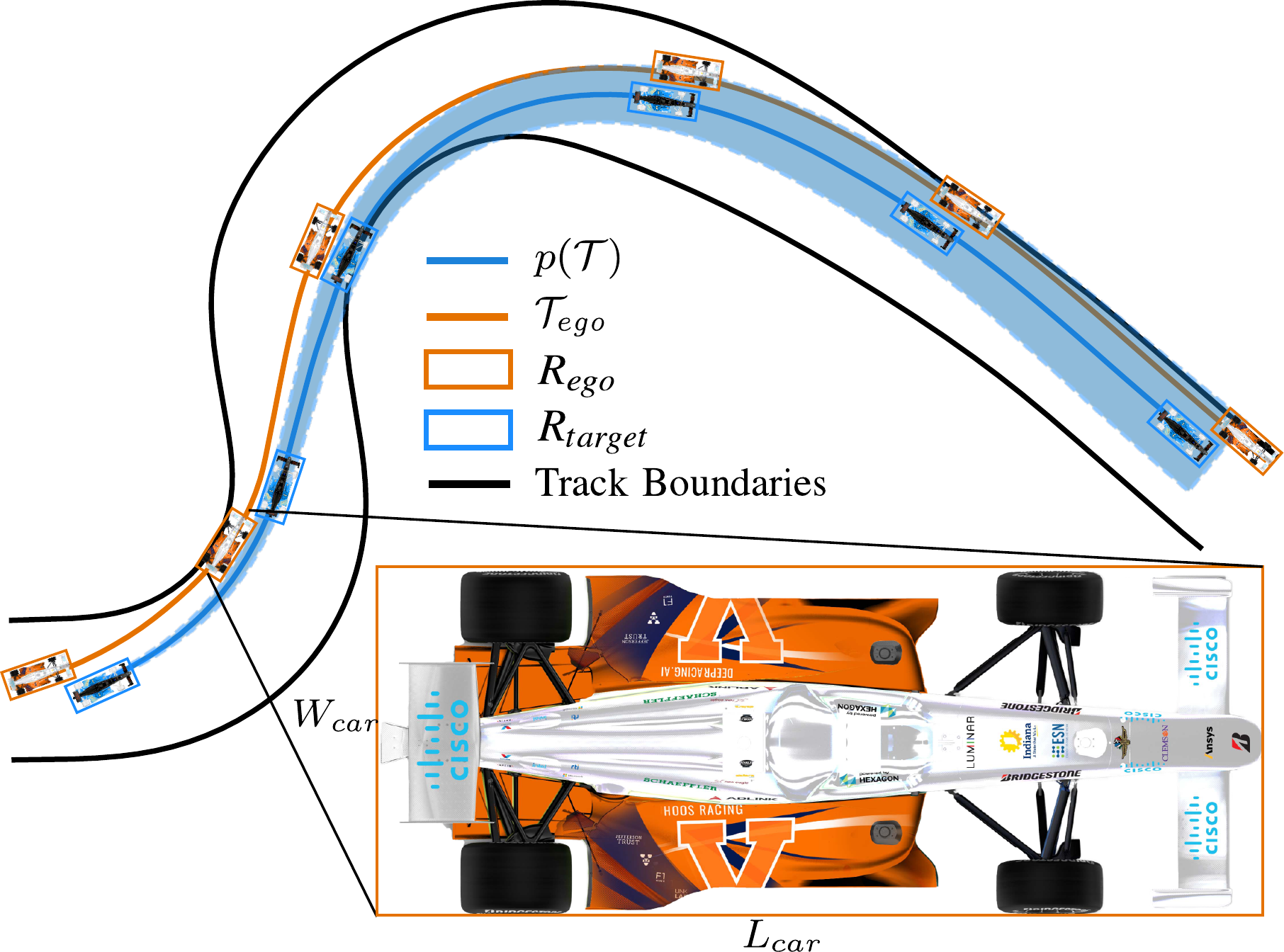}
   \caption{Problem setting: the ego vehicle’s deterministic trajectory 
   $\mathcal{T}_{ego}$ (orange) is evaluated for collisions against the target vehicle’s uncertain motion 
   $p(\mathcal{T})$ (blue). 
   }
   \label{fig:problem_statement}
\end{figure}



Our goal is to estimate the probability that following \(\mathcal{T}_{ego}\) over \([0,\,T_F]\) will lead to a collision with the target vehicle’s uncertain motion, with \emph{collision} meaning the bounding rectangle of the ego, \(R_{ego}\), intersects that of the target, \(R_{target}\), at any time \(t \in [0,\, T_F].\)  Fig.~\ref{fig:problem_statement} depicts these rectangles, each assumed oriented along its respective heading with width $W_{car}$ and length $L_{car}$.
Formally, let \(\mathbb{T}_{col}\) be the set of all trajectories in \(p(\mathcal{T})\) that collide with \(\mathcal{T}_{ego}\).  
We wish to compute
\begin{equation}
    \Pr(\text{Collision} \,\big\vert\, p(\mathcal{T}),\, \mathcal{T}_{ego})
    \;=\;
    \int_{\mathbb{T}_{col}} p(\mathcal{T}) \,d\mathcal{T}
    \label{eqn:problem_statement_as_math}
\end{equation}

Evaluating this integral directly is infeasible: \(\mathbb{T}_{col}\) spans a continuum of possible trajectories, and even enumerating them is intractable.  
Moreover, any high-dimensional integration over this set lacks a general closed-form solution.  
However, accurate collision-risk assessment is crucial in autonomous racing, where overtaking maneuvers at high speed demand careful modeling of opponent uncertainty and time-varying positions.


We next outline important background information for our approach to transforming this integral (equation \ref{eqn:problem_statement_as_math}) into a tractable problem by modeling collision events with a counting process.

\section{Preliminaries}
\label{sec:background}
\subsection{One-dimensional Gauss--Legendre Quadrature}
\label{subsec:oned_glq}

Gaussian Quadrature~\cite{Gauss_1815} is a method for approximating the integral of a function $f$ by evaluating $f$ at a fixed set of sample points, known as 
\emph{nodes}. 
In particular, it approximates the definite integral
\begin{equation}
    \int_{-1}^1 f(t) \, dt \;\approx\; \sum_{i=1}^n w_i \, f(\xi_i),
    \label{eqn:gauss_quadrature_minus1_to_1}
\end{equation}
where each weight--node pair $(w_i, \,\xi_i)$ is determined by a \emph{Gaussian Quadrature Rule}~\cite{golub_gq_rules}. 
The integer $n$, 
called the \emph{order} of the quadrature, defines the number of nodes.

In this work, we employ a rule called 
\emph{Gauss--Legendre Quadrature}~\cite{jacobi_gl_quadrature} (GLQ). 
An $n$-th order GLQ has nodes at the roots of the $n$-th Legendre polynomial $p_n(u)$:
\begin{equation}
    p_n(u) \;=\; \frac{K_n}{2^n \, n!} \,\frac{d^n}{du^n}\bigl(u^2 - 1\bigr)^n,
    \label{eqn:legendre_polynom_rodrigues_formula}
\end{equation}
where $K_n$ is a constant that normalizes $p_n$ so that $p_n(1) = 1.$ 
The $n$ roots 
$\{\xi_1, \xi_2, \dots, \xi_n\}$ form the GLQ nodes, and the corresponding weights $w_i$ are given by
\begin{equation}
    w_i \;=\; \frac{2}{\!\bigl(1-\xi_i^2\bigr)\,[\,p_n'( \,\xi_i )\,]^2}
    \label{eqn:glq_weights}
\end{equation}

To approximate an integral over a finite interval $[a,b]$, we apply a change of variable $t = \tfrac{b-a}{2}\,\xi + \tfrac{a+b}{2}\,$:
\begin{equation}
    \int_a^b f(t)\,dt 
    \;=\; \int_{-1}^1 f\!\Bigl(\tfrac{b-a}{2}\,\xi + \tfrac{a+b}{2}\Bigr)
    \,\frac{dt}{d\xi}\,d\xi
    \label{eqn:change_of_interval}
\end{equation}
Since $\tfrac{dt}{d\xi} = \tfrac{b-a}{2}$, each node $\xi_i$ maps to
\begin{equation}
    t_i \;=\; \tfrac{b-a}{2}\,\xi_i \;+\; \tfrac{a+b}{2},
    \label{eqn:gauss_quadrature_x_of_xi}
\end{equation}
And thus
\begin{equation}
    \int_a^b f(t)\,dt 
    \;\approx\; \tfrac{b-a}{2}\,\sum_{i=1}^n w_i \,f\bigl(t_i\bigr)
    \label{eqn:gauss_quadrature_a_b}
\end{equation}
This approximation will serve as a core building block for our collision risk estimation method.

\subsection{Two-dimensional Gauss--Legendre Quadrature}

GLQ naturally extends to two-dimensional integrals. 
Let $g : \mathbb{R}^2 \rightarrow \mathbb{R}$ be a scalar function, and consider 
the double integral of $g$ over a rectangular region 
$R = [x_1, x_2] \times [y_1, y_2]$:
\begin{equation}
    \int_R g(x,y) \,dA 
    \;=\; \int_{y_1}^{y_2}\!\!\int_{x_1}^{x_2} g(x,y)\,dx\,dy
    \label{eqn:glq_2d}
\end{equation}
First, we apply GLQ of order $m$ with respect to $x$, treating $y$ as a constant in each inner integral. Defining 
$\chi_i = \tfrac{x_2 - x_1}{2}\,\xi_i + \tfrac{x_1 + x_2}{2}$, we get
\begin{equation}
    \int_R g(x,y)\,dA 
    \;\approx\; \frac{x_2 - x_1}{2} \,\int_{y_1}^{y_2} \sum_{i=1}^m
    w_i\,g\!\bigl(\chi_i,\,y\bigr)\,dy
    \label{eqn:glc_x}
\end{equation}
Next, we apply GLQ again in $y$. Let 
$\psi_j = \tfrac{y_2 - y_1}{2}\,\xi_j + \tfrac{y_1 + y_2}{2}$ 
Then
\begin{equation}
    \int_R g(x,y)\,dA 
    \;\approx\; \frac{x_2 - x_1}{2}\,\frac{y_2 - y_1}{2}
    \,\sum_{j=1}^m \Biggl[\,w_j \sum_{i=1}^m w_i\, g(\chi_i,\,\psi_j)\Biggr]
    \label{eqn:glc_y}
\end{equation}
Because $w_j$ is constant with respect to the inner sum:
\begin{equation}
    \int_R g(x,y)\,dA 
    \;\approx\; \frac{\bigl(x_2 - x_1\bigr)\,\bigl(y_2 - y_1\bigr)}{4}
    \sum_{j=1}^m \sum_{i=1}^m w_i\,w_j\, g\bigl(\chi_i,\,\psi_j\bigr)
    \label{eqn:full_glc}
\end{equation}
This approach is sometimes called \emph{Gaussian Cubature}~\cite{cubature_formulae1,cubature_formulae2}; 
we call the two-dimensional case based on Legendre polynomials 
\emph{Gauss--Legendre Cubature (GLC)}. 

GLC plays a central role in our collision-risk estimation, providing a numerical approximation of collision events at fixed instants in time. 
In the next subsection, we describe how a counting-process formulation leverages these instantaneous collision estimates to compute an overall collision probability across time.

\subsection{Non-homogeneous Poisson Processes}
\label{subsec:nhpp}

We briefly recall the definition of a Non-Homogeneous Poisson Process (NHPP). 
\begin{definition}[NHPP]
\label{def:nhpp}
A counting process $N(t)$ is called a \emph{Non-Homogeneous Poisson Process} 
on $[0,\,T_F]$ if:
\begin{enumerate}
    \item $N(0)=0$,
    \item For all $t,\Delta t>0$, the increment 
    $N\bigl(t + \Delta t\bigr) - N(t)$ follows a Poisson distribution with 
    mean $\int_{t}^{t+\Delta t}\!\lambda(\tau)\,d\tau$, where $\lambda(t)$ is a 
    time-varying hazard function.
\end{enumerate}
\end{definition}

\begin{lemma}[Zero-event probability in an NHPP]
\label{lemma:nhpp_zero_events}
Let $N(t)$ be an NHPP with hazard function $\lambda(t)$ on $[0,T_F]$. Then 
the probability that $N(t) = 0$ for all $t\in[0,\,T_F]$ is 
\begin{equation}
    \Pr\!\bigl[N(T_F)=0\bigr] 
    \;=\; \exp\!\Bigl[-\!\int_{0}^{T_F}\!\lambda(t)\,dt\Bigr].
    \label{eqn:pr_NTF_eq_0}
\end{equation}
Hence, the probability of at least one event by time $T_F$ is
\begin{equation}
    \Pr\!\bigl[N(T_F)\ge 1\bigr] 
    \;=\; 1 \;-\; \exp\!\Bigl[-\!\int_{0}^{T_F}\!\lambda(t)\,dt\Bigr].
    \label{eqn:nhpp_PrNeq0}
\end{equation}
\end{lemma}

\noindent
These results follow directly from Poisson-process theory~\cite{ross1995stochastic}. 
In the collision-risk context, we interpret $N(t)$ as the cumulative number of collisions between the ego and target vehicles on $[0,\,T_F]$. By choosing an appropriate hazard function, $\lambda(t)$, tied to the \emph{instantaneous} collision probability, we can compute the overall collision risk via Lemma~\ref{lemma:nhpp_zero_events}, 
as detailed in Section~V.

\section{Gauss–Legendre Rectangle Method}
\label{sec:proposed_method}
\begin{figure*}
    \centering
    \includegraphics[width=0.8\textwidth]{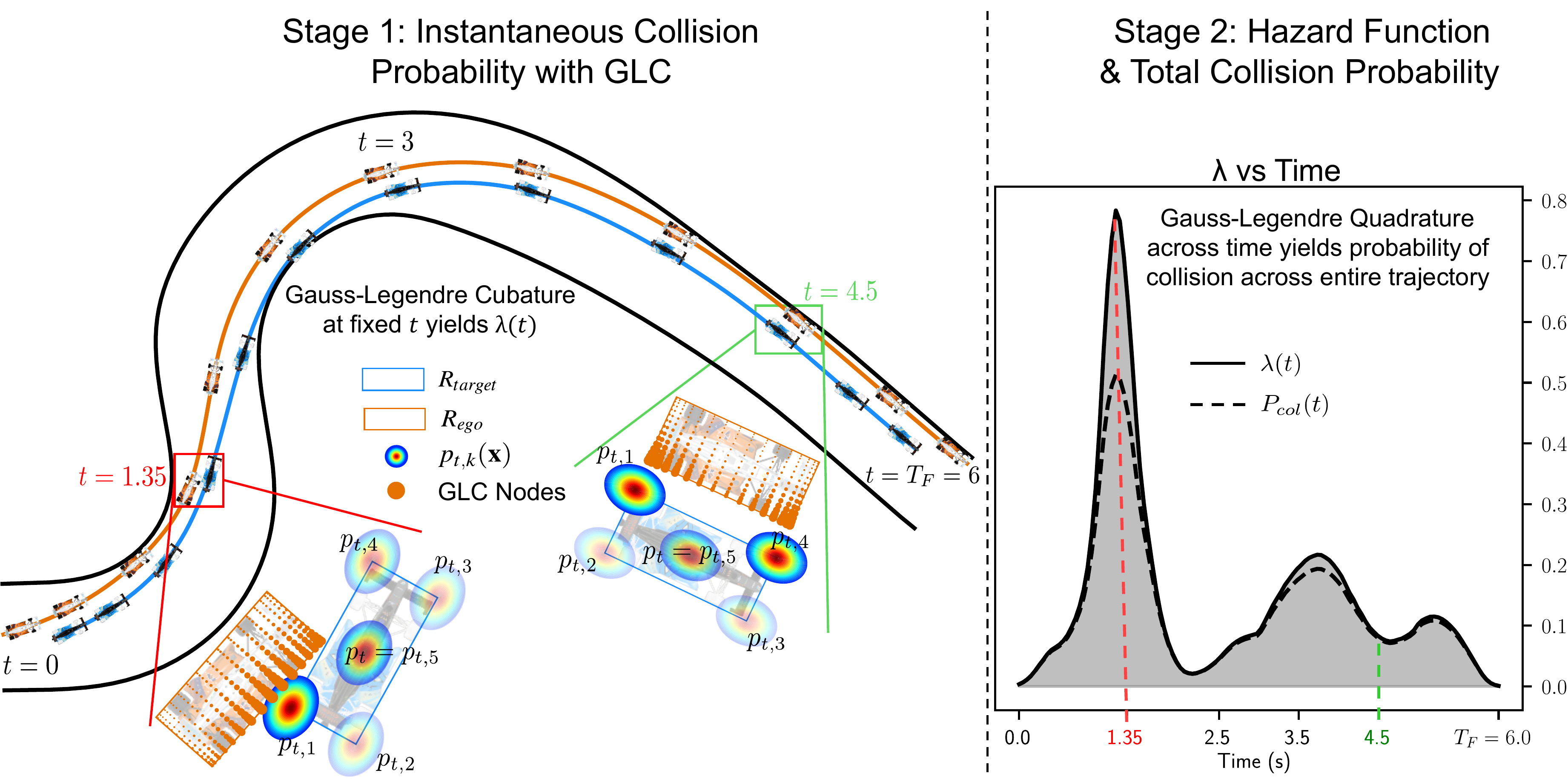}
    \caption{Two stages of the GLR algorithm. Left panel: Instantaneous collision probability $P_{col}(t)$ using GLC. Right panel: Corresponding hazard function $\lambda(t)$; $Pr[N(T_F)=0]$ readily follows from the grey area under the $\lambda(t)$ curve via equation \eqref{eqn:nhpp_PrNeq0}. This example has $T_F=6$ seconds, but the model is extensible to other prediction horizons.}
    \label{fig:pr_no_collision}
    \vspace{-8pt}
\end{figure*}

In this section, we present Gauss–Legendre Rectangle (GLR), our two-stage algorithm for estimating the overall probability of collision between an ego vehicle's planned trajectory, $\mathcal{T}_{ego}$, and an opponent’s stochastic trajectory, $p(\mathcal{T})$, over the prediction horizon $[0,\,T_F]$. 
Our method combines spatial integration using Gauss--Legendre Cubature (GLC) to compute an instantaneous collision probability, with a temporal integration using Gauss--Legendre Quadrature (GLQ) within a Non-Homogeneous Poisson Process (NHPP) framework. The use of an NHPP allows us to integrate the instantaneous collision probabilities over time without resorting to Monte Carlo sampling, significantly reducing computational complexity while accurately capturing the temporal dynamics of collision risk.

\subsection{Stage 1: Instantaneous Collision Probability with GLC}
\label{subsec:stage1_pcol}

At any fixed time $t \in [0,\,T_F]$, let $R_{ego}$ and $R_{target}$ denote the bounding rectangles of the ego and target vehicles (along the respective heading). 
We define $P_{\mathrm{col}}(t)$ as the probability that $R_{target}$ intersects $R_{ego}$. 
Rather than computing this probability directly, we first approximate the probability of no collision, $P_{\lnot \mathrm{col}}(t)$, and then set:
\begin{equation}
    P_{\mathrm{col}}(t) \;=\; 1 - P_{\lnot \mathrm{col}}(t).
    \label{eqn:prob_collision_aem}
\end{equation}

We approximate $R_{target}$ as five independent distributions - with means at the four corners and centroid of the target vehicle - and thus approximate:
\begin{equation}
    P_{\lnot \mathrm{col}}(t) \;\approx\; \prod_{k=1}^{5} \Bigl[\,1 - \!\int_{R_{ego}} p_{t,k}(\mathbf{x})\,dxdy \Bigr],
    \label{eqn:prob_no_collision_fixed_t}
\end{equation}
where $p_{t,k}(\mathbf{x})$ denotes the probability density function (PDF) for the $k^\text{th}$ corner ($k \in \{1,2,3,4\}$) or the centroid ($k=5$) of $R_{target}$ at time $t$. Each of these $p_{t,k}$ distributions come from the same time-indexed $p_t(\mathbf{x})$ and are generated as fixed offsets from $p_t(\mathbf{x})$ at the corners of $R_{target}$ under a rigid body assumption, with $R_{target}$ assumed to be oriented along the target vehicle's direction of travel at $t$.  I.e. each $p_{t,k}(\mathbf{x})$ shares the same covariance as $p_t(\mathbf{x})$ with it's mean moved to the corresponding corner of $R_{target}$. Note that this implies $p_{t,5}(\mathbf{x})=p_t(\mathbf{x})$.  An example of these offset distributions from a sample in our dataset is shown in the Stage 1 section of Figure \ref{fig:pr_no_collision} at $t=1.35$ and $t=4.5$. $R_{target}$ is depicted in blue (alongside each $p_{t,k}(t)$) and $R_{ego}$ in orange for each $t$.  

Each spatial integral in  \eqref{eqn:prob_no_collision_fixed_t} is computed via a Gauss--Legendre Cubature (GLC) of order $n_1$. Moreover, if $W_{\text{car}}$ and $L_{\text{car}}$ are the width and length of the vehicle, and $(w_i,w_j)$ are the corresponding cubature weights with nodes $(\chi_i,\psi_j)$ (cf. Equations~\eqref{eqn:glc_x}--\eqref{eqn:full_glc}), then
\begin{equation}
    \int_{R_{ego}} p_{t,k}(\mathbf{x})\,dxdy \;\approx\; \frac{W_{\text{car}}\,L_{\text{car}}}{4}
    \sum_{i=1}^{n_1}\!\sum_{j=1}^{n_1} w_i\,w_j \, p_{t,k}\!\bigl(\chi_i,\psi_j\bigr).
\end{equation}

Thus, $P_{\mathrm{col}}(t)$ can be computed at any fixed time $t$. We then utilize this technique of approximating $P_{col}$ to form the basis a Poisson process in Stage~2 of our method. We now describe this second stage.

\subsection{Stage 2: Hazard Function \& Total Collision Probability}
\label{subsec:stage2_hazard}

In our approach, the number of collision events $N(t)$ over $[0,\,T_F]$ is modeled as a Non-Homogeneous Poisson Process (NHPP) with a hazard function $\lambda(t)$. $\lambda(t)$ is derived from $P_{\mathrm{col}}(t)$ and must satisfy two boundary conditions:
\begin{enumerate}
    \item If $P_{\mathrm{col}}(t) \to 0$, then $\lambda(t) \to 0$.
    \item If $P_{\mathrm{col}}(t) \to 1$, then $\lambda(t) \to \infty$.
\end{enumerate}
Property 1 ensures that time intervals with zero instantaneous collision risk do not affect the total collision risk, while Property 2 ensures that guaranteed collision at any fixed time forces the total collision probability to 1. We define a hazard function with both of these properties.

\noindent \textbf{Hazard Function Construction:} 
In survival analysis~\cite{survival_analysis,ross1995stochastic}, the \textit{survival function} $S(t) = \Pr[\text{no collision up to } t]$ has derivative $f(t) = -\tfrac{dS}{dt}$, and the ideal hazard function $\lambda^*(t)$ is
\begin{equation}
\lambda^{\star}(t) = \frac{f(t)}{S(t)} = -\frac{d}{dt} \ln S(t). 
\label{eq:ideal_hazard}
\end{equation}



Where $S(t)$ is the fraction of trajectories in $p(\mathcal{T})$ that have not yet collided with the ego vehicle as of time $t$ and $1-S(t)$ the proportion of trajectories that \emph{have} collided with the ego vehicle as of time $t$.
Intuitively, the hazard function measures how likely a collision is to occur at time~$t$, assuming no collision has happened yet.

Computing $S(t)$ exactly is infeasible, since it requires the joint distribution over all future collision events.
We, therefore, approximate $S(t)$ using the instantaneous collision probability $P_{\text{col}}(t)$ from Eq.~\eqref{eqn:prob_collision_aem}. 
Under a small time-step $\Delta t$, we assume $S(t+\Delta t) \approx S(t)[1 - P_{\text{col}}(t)]$, which yields $-\ln S(t) \approx \tfrac{P_{\text{col}}(t)}{1 - P_{\text{col}}(t)}$ ~\cite{kleinbaum1996survival}.



This motivates the following hazard function, $\lambda(t)$, with similar \emph{boundary conditions} to $\lambda^*(t)$: 
\begin{equation}
    \lambda(t) \;=\; \frac{P_{\mathrm{col}}(t)}{1-P_{\mathrm{col}}(t)}
    \label{eqn:hazard}
\end{equation}
Note that our methodology is not constrained to this particular transformation from $P_{\mathrm{col}}(t) \rightarrow \lambda(t)$. 
In preliminary experiments, we explored using a blended hazard rate combining $P_{col}(t)$ and its log-odds transform. 
However, ablation consistently showed that the optimal blending parameter converged to zero, suggesting that the direct log-odds transform Eq.~\eqref{eqn:hazard} suffices.

The following proposition formalizes boundary conditions on $\lambda(t)$ that is shares with $\lambda^*$ and why they are an intuitive and logical basis for our chosen form of $\lambda$.

\begin{proposition}[Boundary Conditions of $\lambda(t)$]
\label{prop:hazard_boundary}
If $\lambda(t)$ is defined by \eqref{eqn:hazard}, then:
\begin{enumerate}
   \item $\lambda(t) \to 0$ as $P_{\mathrm{col}}(t) \to 0$, and 
   \item $\lambda(t) \to \infty$ as $P_{\mathrm{col}}(t) \to 1$.
\end{enumerate}
\end{proposition}

\begin{proof}
When $P_{\mathrm{col}}(t)$ is near zero, both terms in \eqref{eqn:hazard} are negligible; when $P_{\mathrm{col}}(t)$ approaches 1, the second term diverges due to  $1-P_{\mathrm{col}}(t)$ tending to zero. 
\end{proof}
Thus, the total collision probability rises to 1 whenever a collision is certain at any instant $t$. Additionally, sections of the prediction horizon with very small $P_{\mathrm{col}}(t)$ do not contribute any additional total collision probability.\\
\textbf{Justification for choice of $\lambda(t)$}: 
While $P_{\mathrm{col}}(t)$ may exhibit temporal dependence, the log-odds form of $\lambda(t)$ provides a computationally efficient surrogate for an NHPP with independent increments. 
We show empirically in Section ~\ref{sec:results} that this approach to the hazard rate $\lambda$ is remarkably effective for racing and leave more complex formulations of $\lambda$, including those based on learned models, as future work. 

\noindent \textbf{Final Computation via GLQ:}

In accordance with Definition~\ref{def:nhpp} and Lemma~\ref{lemma:nhpp_zero_events} from Section~IV, the probability of no collisions is as follows:
\begin{equation}
    \Pr\!\bigl[N(T_F)=0\bigr] \;=\; \exp\!\Bigl[-\int_{0}^{T_F}\lambda(t)\,dt]\Bigr]
    \label{eqn:nhpp_PrNeq0}
\end{equation}
Thus, the overall collision probability is
\begin{equation}
    \Pr\!\bigl[\text{Collision}\bigr] \;=\; 1 - \exp\!\Bigl[-\int_{0}^{T_F}\lambda(t)\,dt\Bigr]
    \label{eqn:overall_no_collision_prob}
\end{equation}
We sample times $\{t_1,\dots,t_{n_2}\}$ and weights $\{w_1,\dots,w_{n_2}\}$, then approximate the time integral in \eqref{eqn:nhpp_PrNeq0} using Gauss--Legendre Quadrature (GLQ) of order $n_2$, with nodes $\{t_i\}$ and weights $\{w_i\}$:
\begin{equation}
    \int_{0}^{T_F}\lambda(t)\,dt \;\approx\; \frac{T_F}{2}\sum_{i=1}^{n_2} w_i\,\lambda\!\bigl(t_i\bigr)
\end{equation}
The overall collision risk then readily follows from \eqref{eqn:overall_no_collision_prob}. 
\begin{figure}
    \centering
    \includegraphics[width=0.825\columnwidth]{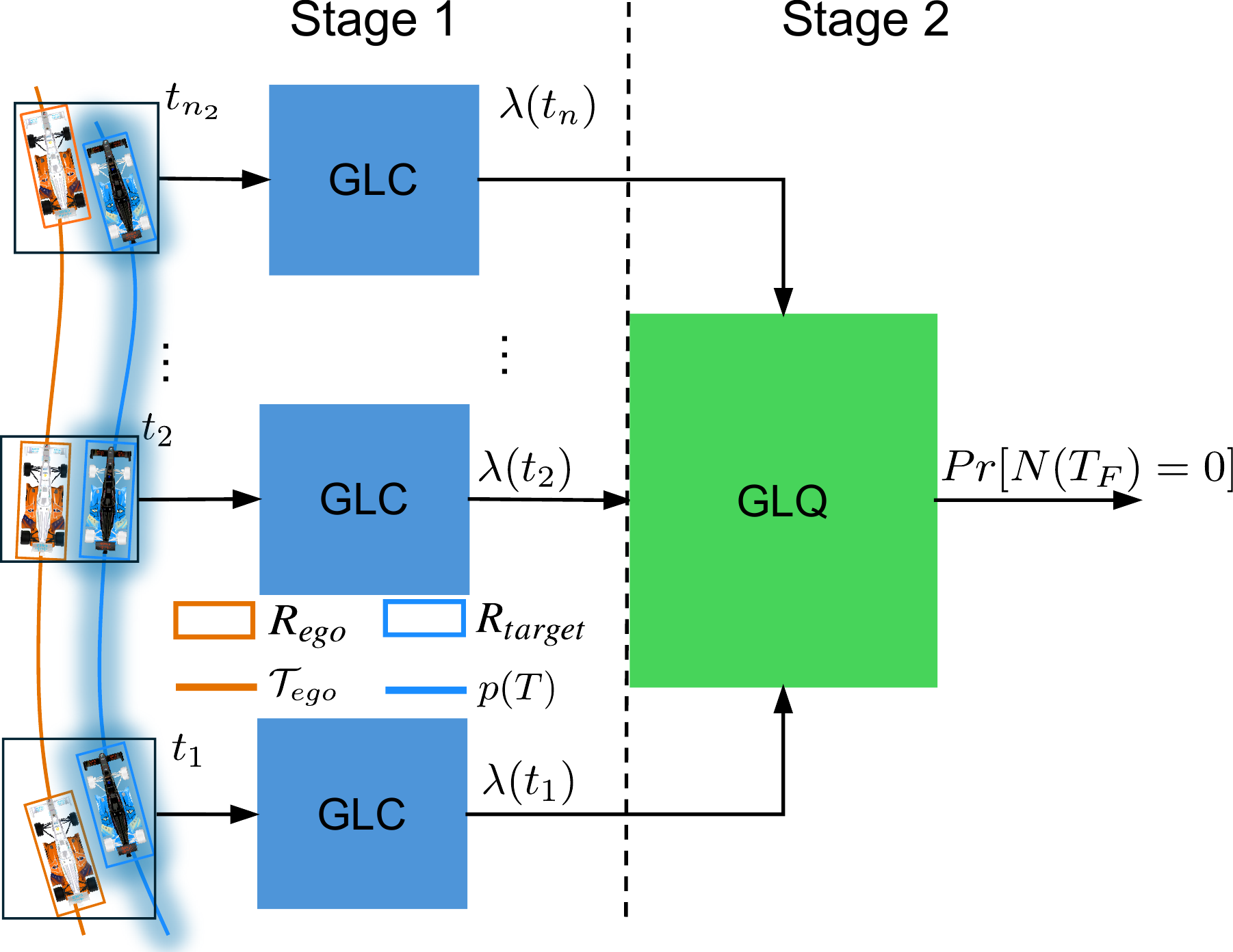}
    \caption{Flowchart of our method. Stage\,1 (GLC) computes $P_{\mathrm{col}}(t)$ and 
    $\lambda(t)$ at discrete times; Stage\,2 (GLQ) integrates $\lambda(t)$ to yield
    overall collision probability.}
    \label{fig:flowchart}
\end{figure}
\noindent \textbf{Error and Convergence:}
Increasing $n_{1}$ or $n_{2}$ improves the accuracy of the GLC and GLQ approximations, respectively. They can be used to balance computational speed and estimation fidelity.

\noindent \textbf{GLR Summary:}
GLR thus consists of two stages:
\begin{enumerate}
    \item \textbf{Stage 1 (GLC):} For each sampled time $t_i$, compute $P_{\mathrm{col}}(t_i)$ 
    using GLC (see Equations~\eqref{eqn:prob_collision_aem} --\eqref{eqn:prob_no_collision_fixed_t})
    and derive $\lambda(t_i)$ via \eqref{eqn:hazard}.
    \item \textbf{Stage 2 (GLQ):} Integrate $\lambda(t)$ over $[0,\,T_F]$ to obtain the overall collision probability using \eqref{eqn:overall_no_collision_prob}.
\end{enumerate}

Figures~\ref{fig:pr_no_collision} and~\ref{fig:flowchart} show our two-stage algorithm. Stage\,1 derives the instantaneous collision probability $P_{\mathrm{col}}(t)$, and by extension $\lambda(t)$, at any fixed time $t$, and Stage\,2 integrates this over $[0,\,T_F]$ via a Non-Homogeneous Poisson Process (NHPP), producing the overall collision risk.
Algorithm~\ref{algo1} provides the pseudocode for this procedure. This algorithm is very well suited for parallel computation. Each of the Stage 1 GLC computations at the Stage 2 GLQ nodes is independent of the others. Additionally, all evaluations of the various $p_k$ in \eqref{eqn:prob_no_collision_fixed_t} can be done simultaneously on a parallel platform.  

\begin{algorithm}
\caption{The overall GLR Algorithm}\label{alg:main}
\begin{algorithmic}[1]
\Require $\mathcal{T}_{ego}$, $p(\mathcal{T})$
\State$\boldsymbol{\chi} \gets \{\mathbf{x}_1, \mathbf{x}_2, ..., \mathbf{x}_{{n_1}^2}\}$ \Comment{Nodes for Stage 1 GLC}
\State$\mathbf{t}\gets\{t_1, t_2, ..., t_{n_2}\}$ \Comment{Nodes for Stage 2 GLQ}
\State$\boldsymbol{\Lambda} \gets [\text{ }]$ \Comment{List of $\lambda$ values at each $t_i$}
\State Stage 1: Compute $\lambda$ for each $t_i$ with GLC
\For{$i \in \{[1,2,...,n_2]\}$}
    \State $R_{ego} \gets Rect(\mathcal{T}_{ego}(t))$\Comment{Bounding rectangle}
    \vspace{2pt}
    \State $P_{col}(t_i) \gets GLC(p_{t_i}(\mathbf{x}), R_{ego}, \boldsymbol{\chi})$  \Comment{Stage 1 GLC}
    \vspace{2pt}
    \State$\boldsymbol{\Lambda}.append\Big(\frac{P_{col}(t_i)}{1-P_{col}(t_i)}\Big)$
\EndFor
    \State Stage 2: GLQ to estimate Pr[Collision]
    \State$Pr(\text{No Collision}) \gets \text{exp}[-\frac{T_F}{2}\sum_{k=1}^{n_2}w_k\Lambda_k]$
    \State$Pr(\text{Collision }\vert\text{ }p(\mathcal{T}), \mathcal{T}_{ego}) \gets 1 - Pr(\text{No Collision})$

\end{algorithmic}
\label{algo1}

\end{algorithm}

\section{Experimental Results}
\label{sec:results}
\subsection{Dataset Collection and Description}
\label{subsec:dataset}

Because no publicly available dataset exists specifically for autonomous racing overtaking scenarios with complete trajectory data, we curated a custom dataset using the DeepRacing simulation environment~\cite{deepracing_date,weiss2020deepracing}. 
Using an automated extraction script, we processed raw F1 video game telemetry to identify overtaking maneuvers where one vehicle gained a race position (e.g., moving from third to second place) at the expense of another (e.g., dropping from second to third).
Such a change in position is depicted in Figure~\ref{fig:example}.
This procedure yielded 446 distinct overtaking scenarios collected across five diverse racing tracks — Australia, Jeddah, Silverstone, Monza, and Bahrain - each capturing a high-speed interaction between an ego vehicle and a defending target vehicle.
Each scenario includes the positions, velocities, and timestamps of both the Ego and Target up to a 6.0s prediction horizon, sampled at 100Hz.  
This horizon was chosen to capture the full duration of most real-world overtaking maneuvers in autonomous racing. 
While simulation-based, this dataset reflects realistic racing behavior and track geometries, with scenarios spanning a range of approach speeds, lateral offsets, and opponent trajectories. It provides a valuable testbed for evaluating both the accuracy and conservatism of risk estimation algorithms.
To preserve double-blind review, we omit the GitHub repository link; however, all code and data will be released upon publication to support reproducibility and benchmarking.


\subsection{Experimental Setup}
\label{subsec:numerical_setup}

\begin{figure}[b]
    \includegraphics[width=0.925\columnwidth]{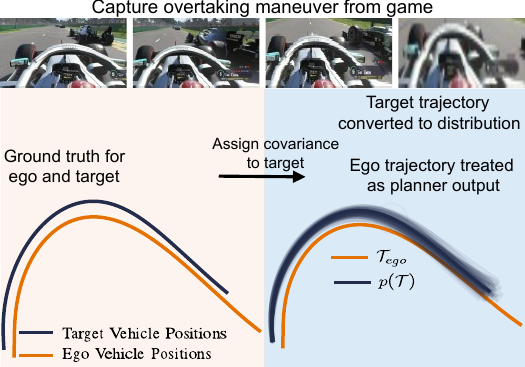}
    \caption{
    An example overtaking scenario from our F1 dataset.  $\mathcal{T}_{ego}$ is shown with some samples from $p(\mathcal{T})$
    }
    \label{fig:example}
\end{figure}
We interpret the Ego's recorded trajectory $\mathcal{T}_{ego}$, as the motion planner's output and fit it, via least squares, with a seventh  order B\'ezier curve in $\mathbb{R}^2$. 
This is a reasonable proxy, as any deterministic 
planner output can be approximated by such a spline.  
However, for the target vehicle, we require a distribution of possible trajectories $p(\mathcal{T})$, rather than a single deterministic one. 
Accordingly, we use probabilistic B\'ezier curves (PBC)~\cite{pbc, dbf} which treat each B\'ezier control point as a normal distribution, and inherently captures increasing uncertainty over time. 
This is implemented by assigning the recorded trajectory of the target vehicle as the mean of the PBC and setting covariance that minimize the the average KL divergence between $p_t(\mathbf{x})$ and a target Gaussian whose covariance expands linearly from $0.1^2\mathbf{I}$ to $\mathbf{I}$. 
Notably, our method does not depend on this modeling choice; any suitable target-vehicle
distribution $p(\mathcal{T})$ from existing prediction frameworks~\cite{mtr++,BARTe,densetnt} can be used.  A probabilistic model of vehicle dynamics could also be used to generate $p(\mathcal{T})$, which may lead to non-Gaussian $p_t(\mathbf{x})$ pdfs for long prediction horizons. The specifics of what prediction methodology produces $p(\mathcal{T})$ is out-of-scope for this work and we use this PBC model as a means of evaluating GLR's ability to assess collision risk.
Figure~\ref{fig:example} illustrates this final step: each recorded target trajectory is fitted by a PBC to form $p(\mathcal{T})$, providing the distribution necessary for our probabilistic collision risk estimation.



\noindent \textbf{Ground-Truth Collision Probability:}
We generate ground-truth collision probabilities via Dense Monte Carlo sampling: 2000 trajectories are drawn from $p(\mathcal{T})$ and each is checked for collision at 128 uniformly spaced times over the 6.0s horizon. 
The fraction of sampled trajectories that intersect the ego path at any time defines the collision probability. 
Although not suitable in real-time, this dense 
simulation provides a reliable ground truth for evaluating the collision-risk estimates.


\subsection{Evaluation Metrics}
\label{subsec:metrics}

We evaluated collision risk estimation performance using two primary metrics: mean absolute error (MAE) and computation time.
\begin{enumerate}
    \item Mean Absolute Error (MAE): The average absolute difference, across all overtaking scenarios, between estimated probability of collision and that of the Dense Monte Carlo ground-truth
    \item Computation Time (ms) and Loop Rate (Hz): Capture the time for a method to run end-to-end for each overtaking scenario. All techniques were measured on an NVIDIA Quadro RTX 4000 GPU. 
\end{enumerate}
MAE serves as the primary accuracy measure, with lower values indicating better performance. 
Computation time serves as a measure of feasibility, with techniques running slower than $\sim 500 Hz$ generally too slow for racing.
Together, these metrics reflect the tradeoff between accuracy and speed that is critical for real-time autonomous decision-making.
These metrics are consistent with those used in prior work~\cite{discounted_blub}, enabling direct comparison.

\begin{table}[!htbp]
    \centering
    \resizebox{\columnwidth}{!}{%
        \begin{tabular}{|l|c|l|c|}
            \hline
            \textbf{GLR Hyper-Parameter} & \textbf{Value} & \textbf{GLR Hyper-Parameter} & \textbf{Value} \\ \hline
            Stage 1 GLC Order ($n_1$) & 12   & Stage 2 GLQ Order ($n_2$) & 24 \\ \hline
              Prediction Horizon ($T_F$) & 6.0~$\si{\second}$ & Car Length ($L_{car}$) & 5.2~$\si{\meter}$ \\ \hline
            Car Width ($W_{car}$)       & 2.0~$\si{\meter}$ && \\ \hline
        \end{tabular}%
    }
    \caption{GLR Hyperparameters for our experiments.}
    \label{table:hyperparameters}
\end{table}
\subsection{Comparisons to Other Methods}
\label{subsec:benchmarks}

We compare GLR against several representative baselines covering both geometric approximations and temporal risk aggregation methods~\cite{velocity_scaled_pf,risk_density,assumeindependence_quadrotors,discounted_blub,max_circle}.




\textbf{Max-Circle:} A simple geometric baseline that models both vehicles as circles with fixed radius and computes instantaneous overlap probability~\cite{max_circle,satellite_spheres}. We report the maximum risk over time.

\textbf{Risk Density:} A recent approach~\cite{paiola2024riskdensity} that builds a cost field proportional to instantaneous collision likelihood and the radius of the vehicle's bounding sphere.

\textbf{Discounted BIUB:} A variant of the BIUB class of methods~\cite{discounted_blub}, that uses a temporal discount factor to weigh earlier collision probabilities more heavily.


\textbf{Velocity-Scaled Particle Filter (VS-PF):} A BIUB method~\cite{velocity_scaled_pf} that represents the opponent as particles and estimates $P_{\mathrm{col}}(t)$ by summing the mass of colliding particles; total risk is computed as a sum over time.


\textbf{Mutual Independence (MI):} This approach~\cite{assumeindependence_quadrotors} assumes per-timestep collision events are independent and estimates cumulative risk as $1 - \prod_{k} [1 - P_{\mathrm{col}}(t_k)]$, a conservative bound intended for pedestrian safety but overly cautious in racing.





\noindent \textbf{Quasi-Monte Carlo Gauss Legendre, QMLGL}:
An ablation of our method that replaces the full 2D Gauss–Legendre Cubature evaluation (used in Stage 1 of our method) with dense Monte Carlo sampling performed only at the Gauss–Legendre Quadrature nodes used in Stage 2. 
In this ablation, 2000 trajectories are sampled and evaluated at the GLQ nodes to obtain ground-truth $P_{col}(t)$ values. 
Although this baseline increases computation time, it provides an accurate reference for evaluating our collision-risk estimates.

The GLR algorithm and all baselines were implemented in PyTorch, leveraging GPU acceleration for improved performance.
All methods were evaluated on the same NVIDIA Quadro RTX 4000 GPU. The standard deviation of runtime across 446 trials was negligible for all methods (on the order of $<0.1$ ms). 
Table~\ref{table:hyperparameters} lists the hyperparameters used for GLR. $W_{car}$ and $L_{car}$ were taken as width and length of a typical Formula One vehicle. $T_F$ was chosen to match the duration of an overtaking maneuver. The GLQ/GLC orders, $n_1$ and $n_2$, were chosen as the largest values that still allow a 1000Hz loop rate.  
\subsection{Empirical Results}
\label{subsec:actual_results}
We evaluated our GLR algorithm alongside the previously stated comparison methods on all 446 overtaking scenarios. 
Table~\ref{table:mainresults} summarizes the performance in terms of MAE, computation time, and loop rate. The standard deviation ($\sigma$) values are also reported for MAE. 
The Mutual Independence method has the fastest loop rate. However, its high MAE of 0.450 indicates poor accuracy-which can lead to missed overtaking opportunities in high-speed autonomous racing.
\begin{table}
    \centering
    \renewcommand{\arraystretch}{1.2}
    \resizebox{\columnwidth}{!}{%
    \begin{tabular}{|l|c|c|c|c|}
        \hline
        \rowcolor{lightgray} Method & MAE $\pm\text{ }\sigma$ & \makecell{Computation \\ Time (ms)} & \makecell{Loop \\ Rate (Hz)}  \\ \hline
        GLR (Ours) & .065 $\pm$ .11 & 1 & 1000  \\ \hline
        QMLGL (Ablation of ours) & .058 $\pm$ 0.09& 8 & 125  \\ \hline
        VS-PF ~\cite{velocity_scaled_pf}  & .135 $\pm$ .15 & 1.2 & 833  \\ \hline
        Risk Density ~\cite{risk_density} & .285 $\pm$ .22& .5 & 2000  \\ \hline
        Discounted BIUB ~\cite{discounted_blub} & 0.398 $\pm$ .25& .5 & 2000  \\ \hline
        Mutual Independence ~\cite{assumeindependence_quadrotors} & 0.450 $\pm$ .26& .4 & 2500  \\ \hline
        Max Circle ~\cite{max_circle,satellite_spheres} & 0.477 $\pm$ .26& 0.6 & 1666.6 \\ \hline
        Dense Monte Carlo Ground Truth & N/A & 55.2 & 18.176  \\ \hline
    \end{tabular}
    }
    \caption{Experimental results on our overtaking scenarios dataset. GLR strikes the best balance between accuracy (MAE) and computation time.}
    \label{table:mainresults}
\end{table}

Our proposed GLR algorithm achieves a MAE of $0.065$. Among other methods, the next best method (VS-PF\cite{velocity_scaled_pf}) reports MAE of 0.135. GLR shows a $52.6\%$ improvement against this method while operating at a 1000Hz loop rate. This loop rate makes GLR suitable for integration into realtime racing planners, which operate as fast as $50$Hz~\cite{euroracing}. 
While QMLGL attains slightly lower MAE (0.058), it is an ablation of GLR itself, not a distinct baseline. QMLGL's modest accuracy improvement also comes at a significant computational cost, with GLR operating at eight times the speed of QMLGL (1000 Hz vs. 125 Hz).

These results demonstrate that our NHPP formulation effectively balances prediction accuracy with computational efficiency-yielding reliable collision risk estimates suitable for the stringent demands of autonomous racing.

\textbf{Multi-agent extension:} GLR naturally extends to multi-agent scenarios by summing individual $\lambda(t)$ terms for each opponent, assuming conditional independence. 
This additive formulation enables efficient multi-opponent risk estimation without the exponential complexity of full joint modeling. 
Extending GLR to capture opponent interactions and joint distributions is an important direction for future work.

\section{Conclusion}
\label{sec:conclusion}



This paper presented the Gauss–Legendre Rectangle (GLR) algorithm for estimating probabilistic collision risk, with a focus on high-speed autonomous racing and overtaking maneuvers. 
GLR integrates a Gauss–Legendre cubature method for spatial evaluation with a non-homogeneous Poisson process over time, enabling accurate and efficient risk estimation while maintaining a 1000Hz loop rate.
Compared to several existing baselines, GLR achieves an average improvement of $77\%$ in estimation accuracy. 
Although developed in the context of high-speed autonomous racing, the framework is general and can support a range of motion planning systems requiring continuous risk assessment.
While we employ a survival-analysis-inspired hazard function in this work, the underlying formulation is agnostic to the specific risk model. 
Future work will explore extensions to multi-agent interaction settings and investigate non-parametric hazard models to further enhance fidelity under uncertainty.
Our ongoing work focuses on integrating GLR with a closed-loop motion planner and prediction module, and evaluating its real-time performance in both simulation and real-world autonomous racing scenarios.


\bibliographystyle{unsrt}
\bibliography{references}

\end{document}